\relax
\documentclass[a4]{article} % DO NOT CHANGE THIS
\usepackage{times}  % DO NOT CHANGE THIS
\usepackage{helvet}  % DO NOT CHANGE THIS
\usepackage{courier}  % DO NOT CHANGE THIS
\usepackage[hyphens]{url}  % DO NOT CHANGE THIS
\usepackage{graphicx} % DO NOT CHANGE THIS
\usepackage{caption} % DO NOT CHANGE THIS AND DO NOT ADD ANY OPTIONS TO IT
\DeclareCaptionStyle{ruled}{labelfont=normalfont,labelsep=colon,strut=off} % DO NOT CHANGE THIS
\usepackage{algorithm}
\usepackage{algorithmic}

\usepackage{amsmath}
\usepackage{subcaption}
\DeclareMathOperator*{\argmax}{arg\,max}
\DeclareMathOperator*{\argmin}{arg\,min}

\newcommand{\lif}[0]{
  \leftarrow
}
\newcommand{\rif}[0]{
  \rightarrow
}
\newcommand{\vt}[1]{
\mathbf{#1}
}
\newcommand{\fn}[1]{
{\it #1}
}

\newcommand{\xor}[0]{
  \oplus
}
\newcommand{\fzand}[0]{
  \wedge
}
\newcommand{\bAnd}[0]{
  \bigwedge
}
\newcommand{\fzor}[0]{
  \vee
}
\newcommand{\bOr}[0]{
  \bigvee
}

\newcommand{\fziff}[0]{
  \leftrightarrow
}

\newcommand{\pr}[1]{
  \mathrm{#1}
}

\newcommand{\En}[0]{ % Energy function
  \fn{E}
}
\newcommand{\xdownarrow}[1]{%
  {\left\downarrow\vbox to #1{}\right.\kern-\nulldelimiterspace}
}

\newtheorem{proposition}{Proposition} 

\newtheorem{theorem}{Theorem}
\newtheorem{lemma}{Lemma}

\newtheorem{definition}{Definition}

\newtheorem{example}{Example}

\usepackage{newfloat}

\usepackage{listings}
\floatstyle{ruled}
\newfloat{listing}{tb}{lst}{}
\floatname{listing}{Listing}

\usepackage{bibentry}
\usepackage{authblk}

\begin{document}
\title{Logical Boltzmann Machines}
\author[1]{Son N. Tran}
\author[2]{Artur d'Avila Garcez}
\affil[1]{University of Tasmania (\texttt{sn.tran@utas.edu.au})}
\affil[2]{City, University of London (\texttt{a.garcez@city.ac.uk})}
\date{}

\maketitle
\sloppy

\begin{abstract}
The idea of representing symbolic knowledge in connectionist systems has been a long-standing endeavour which has attracted much attention recently with the objective of combining machine learning and scalable sound reasoning. Early work has shown a correspondence between propositional logic and symmetrical neural networks which nevertheless did not scale well with the number of variables and whose training regime was inefficient. In this paper, we introduce Logical Boltzmann Machines (LBM), a neurosymbolic system that can represent any propositional logic formula in strict disjunctive normal form. We prove equivalence between energy minimization in LBM and logical satisfiability thus showing that LBM is capable of sound reasoning. We evaluate reasoning empirically to show that LBM is capable of finding all satisfying assignments of a class of logical formulae by searching fewer than $0.75\%$ of the possible (approximately 1 billion) assignments. We compare learning in LBM with a symbolic inductive logic programming system, a state-of-the-art neurosymbolic system and a purely neural network-based system, achieving better learning performance in five out of seven data sets. 
\end{abstract}

\section{Introduction}
The research community has been witnessing an increasing attention devoted to the integration of learning and reasoning, especially through the combination of neural networks and symbolic AI \cite{Marcus2018,garcez2020neurosymbolic} into neurosymbolic systems. At the core of a neurosymbolic system there is an algorithm to represent symbolic knowledge in a neural network. One of the goals is to leverage the parallel and distributed properties of the network to perform reasoning. In many neurosymbolic approaches, the most used form of knowledge representation is if-then rules whereby logical reasoning is built upon Modus-Ponens as the only rule of inference \cite{Towel_1994,Franca_2014,Son_2018,Evans_18,Yang_2017,Robin_2018}. Given a formula $B \lif A$ (read ``B if A" following a logic programming notation) a neural network would either infer approximately $B=1$ (True) if $A=1$ by forward chaining, or search for the value of $A$ to confirm or refute the hypothesis $B=1$ (backward chaining). This has two shortcomings. First, Modus-Ponens alone may not capture entirely the power of logical reasoning as required by an application. For example, it may be the case in an application that if $B=0$ (False), the neural network is expected to infer approximately $A=0$ (Modus-Tollens). Second, one may wish to allow other forms of rules to be represented by the neural network such as disjunctive normal form (DNF) with any number of negative literals. 

In this paper, we introduce Logical Boltzmann Machines, a neurosymbolic system that can represent any propositional logic formula in a neural network and achieve efficient reasoning using restricted Boltzmann machines. We introduce an algorithm to translate any logical formula described in DNF into a Boltzmann machine and we show equivalence between the logical formula and the energy-based connectionist model. In other words, we show soundness of the translation algorithm. Specifically, the connectionist model will assign minimum energy to the assignments of truth-values that satisfy the formula. This produces a new way of performing reasoning in neural networks by employing the neural network to search for the models of the logical formula, that is, to search for assignments of truth-values that map the logical formula to true. We show that Gibbs sampling can be applied efficiently for this search with a large number of variables in the logical formula. If the number of variable is small, inference can be carried out analytically by sorting the free-energy of all possible truth-value assignments. Returning to our example with formula $B \lif A$, Logical Boltzmann Machines can infer approximately $B=1$ given $A=1$, and it can infer approximately $A=0$ given $B=0$ since both truth-value assignments - ($A=True$, $B=True$) and ($A=False$, $B=False$) - would minimise the energy of the network. 

In what concerns the representational issue of rules other than if-then rules, we propose a new way of converting any logical formula into strict DNF (SDNF) which is shown to map conveniently onto Restricted Boltzmann Machines (RBMs). In the experiments reported in this paper, this new mapping into SDNF and RBMs is shown to enable approximate reasoning with a large number of variables. The proposed approach is evaluated in a logic programming benchmark task whereby machine learning models are trained from data and background knowledge. Logical Boltzmann Machines achieved a better training performance (higher test set accuracy) in five out of seven data sets when evaluated empirically on this benchmark in comparison with a purely-symbolic learning system (Inductive Logic Programming system Aleph \cite{aleph}), a neurosymbolic system for Inductive Logic Programming (CILP++ \cite{Franca_2014}) and a purely-connectionist system (standard RBMs \cite{Smolensky_1995}).

The contribution of this work is twofold:
\begin{itemize}
    \item A theoretical proof to equivalently map logical formulas and probabilistic neural networks, namely Restricted Boltzmann Machines, which can facilitate neuro-symbolic learning and reasoning.
    \item A foundation for the employment of statistical inference methods to perform logical reasoning.
\end{itemize}
The remainder of the paper is organised as follows. In the next section, we review the related work. Section \ref{theory} describes and proves correctness of the mapping from any logical formula into SDNF and then RBMs. Section \ref{sub:reasoning} defines reasoning by sampling and energy minimization in RBMs. Section \ref{LBMsection} introduces the LBM system and evaluates scalability of reasoning with an increasing number of variables. Section \ref{exp} contains the experimental results on learning and the comparison with a symbolic, neurosymbolic and a purely-neural learning system. We then conclude the paper and discuss directions for future work. 

\section{Related Work}
\label{sec:background}
One of the earliest work on the integration of neural networks and symbolic knowledge is Knowledge-based Artificial Neural Network \cite{Towel_1994} which encodes {\it if-then} rules into a hierarchy of perceptrons. In another early approach \cite{Garcez_2001}, a single-hidden layer neural network with recurrent connections is proposed to support logic programming rules. An extension of that approach to work with first-order logic programs, called CILP++ \cite{Franca_2014}, uses the concept of {\it propositionalisation} from Inductive Logic Programming (ILP) whereby first-order variables can be treated as propositional atoms in the neural network. Also based on first-order logic programs, \cite{Evans_18} propose a differentiable ILP approach that can be implemented by neural networks, while \cite{Cohen_2017} maps stochastic logic programs into a differentiable function also trainable by neural networks. These are all supervised learning approaches.

Among unsupervised learning approaches, Penalty Logic \cite{Pinkas_1995} was the first work to integrate propositional and non-monotonic logic formulae into symmetric networks. However, it required the use of higher-order Hopfield networks which can become complicated to construct and inefficient to train with the learning algorithm of Boltzmann machines (BM). Such higher-order networks require transforming the energy function into a quadratic form by adding hidden variables not present in the original logic formulae for the purpose of building the network. More recently, several attempts have been made to extract and encode symbolic knowledge into RBMs \cite{Leo_2011,Son_2018}. These are based on the structural similarity between symmetric networks and bi-conditional logical statements and do not contain soundness results. By contrast, and similarly to Penalty Logic, the approach introduced in this paper is based on a proof of equivalence between the logic formulae and the symmetric networks, but without requiring higher-order networks.

Alongside the above approaches which translate from a symbolic to a neural representation, normally from if-then rules to a feedforward or recurrent neural network, there are also hybrid approaches, which combine neural networks with symbolic AI systems and logic operators, such as DeepProbLog \cite{Robin_2018} and Logic Tensor Networks (LTN) \cite{Serafini_2016}. While DeepProbLog adds a neural network module to probabilistic logic programming, LTN represents the level of truth of first-order logic statements in the neural network. LTN employs a discriminative approach to infer the level of truth rather than the generative approach adopted in this paper by LBM. Although a discriminative approach can be very useful for evaluating assignments of variables to truth-values, we argue that it is not adequate for the implementation of a search for the satisfying assignments of logical formulae. For this purpose, the use of a generative approach is needed as proposed in this paper with LBM.

\section{Knowledge Representation in RBMs}
\label{theory}
%\subsection{Propositional Calculus and RBMs}
\label{dnf_rbm}
An RBM \cite{Smolensky_1995} can be
seen as a two-layer neural network with bidirectional (symmetric) connections,
which is characterised by a function called the energy of the RBM:

{
\begin{equation}
\label{eq:rbm_en}
  \En(\vt{x},\vt{h}) = -\sum_{i,j} w_{ij}x_ih_j - 
\sum_{i}a_ix_i - \sum_jb_jh_j
\end{equation}
}

\noindent where $a_i$ and $b_j$ are the biases of input unit $x_i$ and hidden unit $h_j$, respectively, and $w_{ij}$ is the connection weight between $x_i$ and $h_j$. This RBM represents a joint distribution $p(\vt{x},\vt{h})= \frac{1}{Z}e^{-\frac{1}{\tau}\En(\vt{x},\vt{h})}$ where $Z=\sum_{\vt{x}\vt{h}}e^{-\frac{1}{\tau}\En(\vt{x},\vt{h})}$ is the partition function and parameter $T$ is called the temperature of the RBM, $\vt{x} = \{x_i\}$ is the set of visible units and $\vt{h} = \{h_j\}$ is the set of hidden units in the RBM.

In propositional logic, any well-formed formula (WFF) $\varphi$ can be mapped into Disjunctive
Normal Form (DNF), i.e. disjunctions ($\vee$) of conjunctions ($\fzand$), as follows:
%\vskip -.2cm
{%\footnotesize %\scriptsize
\begin{equation*}
\varphi \equiv \bOr_j (\bAnd_{t \in \mathcal{S}_{T_j}} \pr{x}_t \fzand \bAnd_{k \in \mathcal{S}_{K_j}} \neg \pr{x}_{k})
\end{equation*}
}
%\vskip -.3cm
\noindent where $(\bAnd_{t \in \mathcal{S}_{T_j}} \pr{x}_t \fzand \bAnd_{k \in
\mathcal{S}_{K_j}} \neg \pr{x}_{k})$ is called a conjunctive clause, e.g. $\pr{x}_1 \fzand \pr{x}_2 \fzand \neg \pr{x}_3$. Here, we
denote the propositional variables (literals) as $\pr{x}_t$ for positive literals (e.g. $\pr{x}_1$), $\pr{x}_k$ for negative literals (e.g. $\neg \pr{x}_3$); $\mathcal{S}_{T_j}$ and $\mathcal{S}_{K_j}$ denote, respectively, the sets of $T_j$ indices of the positive literals and $K_j$ indices of the negative literals in the formula. This notation may seem over-complicated but it will be useful in the proof of soundness of our translation from SDNF to RBMs.   

\begin{definition}
Let $s_\varphi(\vt{x}) \in \{0,1\}$ denote the truth-value of a WFF $\varphi$ given an assignment of truth-values $\vt{x}$ to the literals of $\varphi$ with truth-value $True$ mapped to 1 and truth-value $False$ mapped to 0. Let $\En(\vt{x},\vt{h})$ denote the energy function of an energy-based neural network $\mathcal{N}$ with visible units $\vt{x}$ and hidden units $\vt{h}$. $\varphi$ is said to be \emph{equivalent} to $\mathcal{N}$ if and only if for any assignment $\vt{x}$ there exists a function $\psi$ such that $s_\varphi(\vt{x}) = \psi(\En(\vt{x},\vt{h}))$.
%where $A>0$ and $B$ are constants and $\En_{rank}(\vt{x}) = min_\vt{h}\En(\vt{x},\vt{h})$ is the rank of the energy function of $\mathcal{N}$ minimised over all hidden units $\vt{h}$.

\end{definition} 
% Explain this equivalance
This definition of equivalence is similar to that of
 Penalty Logic \cite{Pinkas_1995}, whereby all assignments of truth-values satisfying a WFF $\varphi$ are mapped to global minima of the energy function of network $\mathcal{N}$. In our case, by construction, assignments that do not satisfy the WFF are mapped to maxima of the energy function.

\begin{definition}
A \emph{strict DNF} (SDNF) is a DNF with at most one conjunctive clause that maps to $True$ for any assignment of truth-values $\vt{x}$. 
A \emph{full DNF} is
    a DNF where each propositional variable must appear at least once in every
    conjunctive clause. 
\end{definition}

\begin{lemma}
\label{lem:wff2ecs}
Any SDNF $\varphi \equiv \bOr_j (\bAnd_{t \in \mathcal{S}_{T_j}}
\pr{x}_t \fzand \bAnd_{k \in \mathcal{S}_{K_j}} \neg \pr{x}_{k})$ can
be mapped onto an energy function:
 
\begin{equation*}
\En(\vt{x}) = -\sum_{j} (\prod_{t \in \mathcal{S}_{T_j}} x_t \prod_{k \in \mathcal{S}_{K_j}} (1 - x_{k}))
\end{equation*}

\noindent where $\mathcal{S}_{T_j}$ (resp. $\mathcal{S}_{K_{j}}$) is the set of
$T_j$ (resp. $K_j$) indices of the positive (resp. negative) literals in $\varphi$.
\end{lemma}
\begin{proof}
%By definition, $\varphi \equiv \bOr_j (\bAnd_{t\in \mathcal{S}_{T_j}}
%\pr{x}_t \fzand \bAnd_{k \in \mathcal{S}_{K_j}} \neg
%\pr{x}_{k})$. 
Each conjunctive clause $\bAnd_{t\in \mathcal{S}_{T_j}}
\pr{x}_t \fzand \bAnd_{k \in \mathcal{S}_{K_j}} \neg \pr{x}_{k}$ in $\varphi$
can be represented by $\prod_{t\in \mathcal{S}_{T_j}} x_t \prod_{k\in
  \mathcal{S}_{K_j}} (1-x_{k})$ which maps to $1$ if and only if
$x_t=1$ (i.e. $True$) and $x_{k}=0$ (i.e. $False$) for all
$t \in \mathcal{S}_{T_j}$ and $k \in \mathcal{S}_{K_j}$. Since
$\varphi$ is a SDNF, it is $True$ if and only if one conjunctive
clause is $True$. Then, the sum $\sum_{j}( \prod_{t\in
  \mathcal{S}_{T_j}} x_t \prod_{k \in \mathcal{S}_{K_j}} (1-x_{k}))=1$
if and only if the assignment of truth-values to $x_t$, $x_{k}$ is a model of $\varphi$. Hence, the neural network with energy
function $\En=-\sum_{j} (\prod_{t\in \mathcal{S}_{T_j}} x_t \prod_{k
  \in \mathcal{S}_{K_j}} (1-x_k))$ is such that $s_\varphi(\vt{x}) = -
\En(\vt{x})$.
\end{proof}
\begin{table}[ht]
    \centering

    \begin{tabular}{|c|c|c||c|c|}
      \hline
      \hline
      $\pr{x}$     &   $\pr{y}$   & $\pr{z}$     & $s_\varphi(\pr{x},\pr{y},\pr{z})$  &  $\En(x,y,z)$\\
      \hline
      $False$ & $False$ & $False$ & $1$              &  $-1$\\
      \hline
      $False$ & $False$ & $True$  & $0$             &  $0$\\
      \hline
      $False$ & $True$ & $False$  & $0$             &  $0$\\
      \hline
      $False$ & $True$  & $True$  & $1$              &  $-1$\\
      \hline
      $True$  & $False$ & $False$ & $0$             &  $0$\\
     \hline
      $True$  & $False$ & $True$  & $1$              &  $-1$\\
      \hline
      $True$  & $True$  & $False$ & $1$              &  $-1$\\
      \hline
      $True$  & $True$  & $True$  & $0$             &  $0$\\
      \hline
    \end{tabular}
    
    \caption{Energy function and truth-table for the formula $((\pr{x} \wedge \neg \pr{y}) \vee (\neg \pr{x} \wedge \pr{y})) \fziff \pr{z}$; we use the symbol $\xor$ to denote exclusive-or, that is $\pr{x} \xor \pr{y} \equiv ((\pr{x} \wedge \neg \pr{y}) \vee (\neg \pr{x} \wedge \pr{y}))$.}
    \label{tab:en_xor_high_bm}
  \end{table}
\begin{example} 
\label{exp:xor2en}
The formula $\varphi \equiv (\pr{x} \xor \pr{y}) \fziff \pr{z}$, c.f. Table \ref{tab:en_xor_high_bm}, can be converted into a SDNF as follows:

\begin{equation*}
\varphi \equiv (\neg \pr{x} \fzand \neg \pr{y} \fzand \neg \pr{z}) \fzor (\neg \pr{x}
\fzand \pr{y} \fzand \pr{z}) \fzor (\pr{x} \fzand \neg \pr{y} \fzand \pr{z}) \fzor (\pr{x} \fzand \pr{y}
\fzand \neg \pr{z})
\end{equation*}

For each conjunctive clause in $\varphi$, a corresponding expression is added to the energy function, e.g. $xy(1-z)$ corresponding to clause $\pr{x} \fzand \pr{y} \fzand \neg \pr{z}$. Hence, the energy function for $N$ equivalent to $\varphi$ becomes:

\begin{equation*}
\begin{aligned}
 \En = &-(1-x)(1-y)(1-z) - xy(1-z) - 
       &x(1-y)z - (1-x)yz
\end{aligned}          
\end{equation*}
 
\end{example}

We now show that any SDNF can be mapped onto an RBM.

\begin{theorem}
\label{theorem:prop_rbm} 
 Any SDNF $\varphi \equiv \bOr_j (\bAnd_{t \in \mathcal{S}_{T_j}}
 \pr{x}_t \fzand \bAnd_{k \in \mathcal{S}_{K_j}} \neg \pr{x}_{k})$ can
 be mapped onto an equivalent RBM with energy function
\begin{equation}
\En(\vt{x},\vt{h}) =
 -\sum_jh_j(\sum_{t \in \mathcal{S}_{T_j}} x_t - \sum_{k \in
   \mathcal{S}_{K_j}}x_{k} - |\mathcal{S}_{T_{j}}| + \epsilon)
\end{equation}

\noindent where $0<\epsilon<1$, $\mathcal{S}_{T_j}$ and  $\mathcal{S}_{K_{j}}$ are, respectively, the sets of indices of the positive and negative literals in each conjunctive clause $j$ of the SDNF, and $|\mathcal{S}_{T_{j}}|$ is the number of positive literals in conjunctive clause $j$.

% $\mathcal{S}_{T_j}$ and $\mathcal{S}_{K_{j}}$ are the sets of $T_j$
% and $K_j$ indices of the positive and negative
% literals in $\varphi$.
%$\mathcal{S}_{T_j}$, $\mathcal{S}_{K_{j}}$ are respectively the sets of indices of positive and negative propositions of each conjunctive clause $j$ in the SDNF; $T_j$ is the size of $\mathcal{S}_{T_{j}}$.

\end{theorem}

\begin{proof} 
We have seen in Lemma \ref{lem:wff2ecs} that any SDNF $\varphi$ can be
mapped onto energy function $\En = -\sum_{j} \prod_{t\in \mathcal{S}_{T_j}}
x_t \prod_{k \in \mathcal{S}_{K_j}} (1-x_{k}) $. %Let us denote $T_j$ as the number of positive literals in a conjunctive clause $j$. 
For each expression
$\tilde{e}_j(\vt{x}) = -\prod_{t \in \mathcal{S}_{T_j}} x_t \prod_{k \in \mathcal{S}_{K_j}}
(1-x_{k})$, we define an energy expression associated with hidden unit
$h_j$ as $e_j(\vt{x},h_j) = -h_j(\sum_{t \in \mathcal{S}_{T_j}} x_t
- \sum_{k \in \mathcal{S}_{K_j}}x_{k} - |\mathcal{S}_{T_{j}}| + \epsilon)$. $e_j(\vt{x},h_j)$ is minimized with value $-\epsilon$ when $h_j=1$, written $min_{h_j}(e_j(\vt{x},h_j)) = -\epsilon$. This is because $-(\sum_{t \in \mathcal{S}_{T_j}}
x_t - \sum_{k \in \mathcal{S}_{K_j}} x_{k} -|\mathcal{S}_{T_{j}}| + \epsilon) = -\epsilon$ if and only if $x_t=1$ and
$x_{k}=0$ for all $t \in \mathcal{S}_{T_j}$ and $k \in \mathcal{S}_{K_j}$.
%$\tilde{e}_j(\vt{x}) = \frac{e_{j \text{ }
%rank}(\vt{x})}{\epsilon}$, where $e_{j \text{ }
%rank}(\vt{x})=min_{h_j}e_j(\vt{x},h_j)$. 
Otherwise, $-(\sum_t x_{t \in \mathcal{S}_{T_j}}
- \sum_{k \in \mathcal{S}_{K_j}} x_{k} -|\mathcal{S}_{T_{j}}| +\epsilon) >0$ and  $min_{h_j}(e_j(\vt{x},h_j)) =
0$ with $h_j=0$. By repeating this process for each
$\tilde{e}_j(\vt{x})$ we obtain that any SDNF $\varphi$ is
equivalent to an RBM with the energy function:
%\vskip -.2cm
{%\footnotesize %\scriptsize
 $
   \label{eq:prop_rbm_en}
   \En(\vt{x},\vt{h}) = -\sum_jh_j(\sum_{t \in \mathcal{S}_{T_j}} x_t - \sum_{k \in \mathcal{S}_{K_j}}x_{k}  -  |\mathcal{S}_{T_{j}}| + \epsilon)  
 $
 }
 %\vskip -.2cm
such that $s_\varphi(\vt{x}) = -\frac{1}{\epsilon} min_{\vt{h}}\En(\vt{x},\vt{h})$.
\end{proof}

%\begin{example}
%The XOR formula $(\pr{x} \xor \pr{y}) \fziff \pr{z}$ can be converted into a SDNF as:

%{%\scriptsize
%\begin{equation*}
%\varphi \equiv (\neg \pr{x} \fzand \neg \pr{y} \fzand \neg \pr{z}) \fzor (\neg \pr{x}
%\fzand \pr{y} \fzand \pr{z}) \fzor (\pr{x} \fzand \neg \pr{y} \fzand \pr{z}) \fzor (\pr{x} \fzand \pr{y}
%\fzand \neg \pr{z})
%\end{equation*}
%}
%\vskip -.2cm
%For each conjunctive clause, for example $\pr{x} \fzand \pr{y} \fzand \neg \pr{z}$ we
%create a term $xy(1-z)$ and add it to the energy function. After all
%terms are added, we have the energy function of a RBM as:
%\vskip -.3cm
%{\scriptsize
%\begin{equation*}
%\begin{aligned}
% \En = &-(1-x)(1-y)(1-z) - xy(1-z) \\
%       &-x(1-y)z - (1-x)yz
%\end{aligned}          
%\end{equation*}
%}

\begin{figure}[ht]
\centering
%\begin{subfigure}{0.4\textwidth}
\includegraphics[width=.5\textwidth]{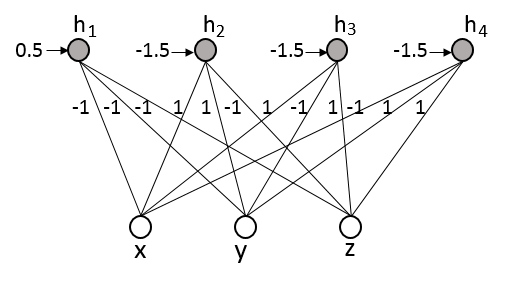}
\caption{An RBM equivalent to the XOR formula $(\pr{x} \xor \pr{y}) \fziff \pr{z}$.}
\label{dnf_xor}
%\label{fig:xor2rbm}
%\vskip -.7cm
\end{figure}

\vspace{0.3cm}
Applying Theorem \ref{theorem:prop_rbm}, an RBM for the XOR formula $\varphi \equiv (\pr{x} \xor \pr{y}) \fziff \pr{z}$ can be built as shown in Figure \ref{dnf_xor}. We choose $\epsilon = 0.5$. The energy function of this RBM is:
%\vskip -.5cm
{%\scriptsize
\begin{equation*} 
  \begin{aligned} \En &= -h_1(-x - y - z + 0.5) - h_2(x + y -  z - 1.5) -\\
  & h_3(x+y-z-1.5) - h_4( -x + y + z - 1.5)
  \end{aligned}
\end{equation*}
}

%Table \ref{tab:en_xor_equip} shows the equivalence between $min_{\vt{h}}\En(\vt{x},\vt{h})$ and the truth-table values of the XOR.
%\begin{table}[ht]
%    \centering
%    {\scriptsize
%    \begin{tabular}{|c|c|c||c|c|}
%      \hline
%      \hline
%      $\pr{x}$     &   $\pr{y}$   & $\pr{z}$     & $s_\varphi(\pr{x},\pr{y},\pr{z})$  &  $min_{\vt{h}}\En(\vt{x},\vt{h})$ \\
%      \hline
%      $0$ & $0$ & $0$ & $True$              &  $-0.5$\\
%      \hline
%      $0$ & $0$ & $1$ & $False$             &  $0$\\
%      \hline
%      $0$ & $1$ & $0$ & $False$             &  $0$\\
%      \hline
%      $0$ & $1$ & $1$ & $True$              &  $-0.5$\\
%      \hline
%      $1$ & $0$ & $0$ & $False$             &  $0$\\
%      \hline
%      $1$ & $0$ & $1$  & $True$              &  $-0.5$\\
%      \hline
%      $1$ & $1$ & $0$ & $True$              &  $-0.5$\\
%      \hline
%      $1$ & $1$ & $1$  &$False$             &  $0$\\
%      \hline
%    \end{tabular}
%    }
%    \caption{The minimised energy function of RBM and the truth table of XOR formula}
%    \label{tab:en_xor_equip}
    %\vskip -.7cm
%  \end{table}

For comparison, one may construct an RBM for XOR using Penalty Logic, as follow. First, we compute the higher-order energy function:
$
\label{eq:xor_en}
\En^p =  4xyz - 2xy - 2xz - 2yz + x +y +z,
$
then we transform it to quadratic form by adding a hidden variable $h_1$ to obtain:
$
\En^p = 2xy - 2xz - 2yz - 8xh_1 - 8yh_1 + 8zh_1 + x + y + z + 12h_1,
$
which is not an energy function of an RBM, so we keep adding hidden variables until the energy function of an RBM might be obtained, in this case:$
\En^p = -8xh_1 - 8yh_1 + 8zh_1 + 12h_1 
    -4xh_2 + 4yh_2 + 2h_2 - 4yh_3 - 4zh_3 + 6h_3 -4xh_4 - 4zh_4 + 6h_4  + 3x + y + z.
$

%which is the RBMs in \fgref{pen_xor}.
%\end{example}

%AG: I really like the above example and comparison with Pinkas. Maybe it can be included in a future journal version of the paper...

The above example illustrates in a simple case the value of using SDNF, in that it produces a direct translation into efficient RBMs by contrast with existing approaches. Next, we discuss the challenges of the conversion of WFFs into SDNF.

\textbf{Representation Capacity:} It is well-known that any formula $\varphi$ can be converted into DNF. If $\varphi$ is not SDNF then by definition there is a group of conjunctive clauses in $\varphi$ which map to $True$ when $\varphi$ is satisfied. This group of conjunctive clauses can always be converted into a full DNF which is also a SDNF. Therefore, any WFF can be converted into SDNF. From Theorem \ref{theorem:prop_rbm}, it follows that any WFF can be represented by the energy of an RBM. For example, $a \vee b$ becomes $(a \wedge \neg b) \vee (\neg a \wedge b) \vee (a \wedge b)$. We now describe a method for converting logical formulae into SDNF, which we use in the empirical evaluations that follow.

Let us consider a clause:
{%\footnotesize
\begin{equation}
\label{eq:cl}
\gamma \equiv \bOr_{t\in \mathcal{S}_T} \neg \pr{x}_t \fzor \bOr_{k\in \mathcal{S}_K} \pr{x}_k
\end{equation}
}
which can be rearranged as $\gamma \equiv \gamma' \fzor \pr{x}'$, where $\gamma'$ is a disjunctive clause obtained by removing $\pr{x}'$ from $\gamma$. $\pr{x}'$ can be either $\neg \pr{x}_t$ or $\pr{x}_k$ for any $t \in \mathcal{S}_T$ and $k \in \mathcal{S}_K$. We have: 
\begin{equation}
\label{conj_ext}
\gamma  \equiv (\neg \gamma' \fzand \pr{x}') \fzor \gamma'
\end{equation}
because $(\neg \gamma' \fzand \pr{x}') \fzor \gamma' \equiv (\gamma' \fzor \neg \gamma') \fzand (\gamma' \fzor \pr{x}') \equiv True \fzand (\gamma' \fzor \pr{x}')$. By De Morgan's law ($\neg( \pr{a}\fzor \pr{b})\equiv \neg \pr{a}\fzand\neg\pr{b}$), we can always convert $\neg \gamma'$ (and therefore $\neg \gamma' \fzand \pr{x}'$) into a conjunctive clause. 

By applying \eqref{conj_ext} repeatedly, each time we can eliminate a variable out of a disjunctive clause by moving it into a new conjunctive clause. The disjunctive clause $\gamma$ holds true if and only if either the  disjunctive clause $\gamma'$ holds true or the conjunctive clause ($\neg \gamma' \fzand \pr{x}'$) holds true. 

As an example, consider the application of the transformation above to an if-then rule (logical implication):

{%\footnotesize %\scriptsize
\begin{equation}
\label{horn}
\pr{y} \lif (\bAnd_{t\in \mathcal{S}_T} \pr{x}_{t} \fzand \bAnd_{k \in \mathcal{S}_K} \neg \pr{x}_k)
\end{equation}
}

% Now apply the idea of
The logical implication is converted to DNF:

{%\footnotesize %\scriptsize
\begin{equation}
\label{logimp_dnf}
(\pr{y} \fzand \bAnd_{t\in \mathcal{S}_T} \pr{x}_t \fzand \bAnd_{k\in \mathcal{S}_K} \neg \pr{x}_k) \fzor (\bOr_{t\in \mathcal{S}_T} \neg \pr{x}_t \fzor \bOr_{k\in \mathcal{S}_K} \pr{x}_k)
\end{equation}
}

\noindent Applying the variable elimination method in \eqref{conj_ext} to all variables in the clause $\bOr_{t\in \mathcal{S}_T} \neg \pr{x}_t \fzor \bOr_{k\in \mathcal{S}_K} \pr{x}_k$, we obtain the SDNF of the logical implication as:

{%\footnotesize%\scriptsize
\begin{equation}
\label{short_horn_dnf}
\begin{aligned}
&(\pr{y} \fzand \bAnd_{t\in \mathcal{S}_T}  \pr{x}_t  \bAnd_{k\in \mathcal{S}_K} \neg \pr{x}_k) 
& \fzor \bOr_{p \in \mathcal{S}_T \cup \mathcal{S}_K} (\bAnd_{t\in\mathcal{S}_T\backslash p} \pr{x}_t \fzand  \bAnd_{k\in\mathcal{S}_K\backslash p} \neg\pr{x}_k \fzand \pr{x}'_{p})
\end{aligned}
\end{equation}
 }
 %\vskip -.2cm
 where $\mathcal{S} .\backslash p$ denotes a set $\mathcal{S}$ from which $p$ has been removed. $\pr{x}'_p\equiv\neg\pr{x}_p$ if $p\in\mathcal{S}_T$. Otherwise, $\pr{x}'_p\equiv\pr{x}_p$. This SDNF only has $|\mathcal{S}_T| + |\mathcal{S}_K| +1$ clauses, making translation to an RBM very efficient. For example, using this method, the SDNF of $\pr{y} \lif (\pr{x}_1 \fzand \pr{x}_2 \fzand \neg \pr{x}_3)$ is $\text{ } (\pr{y} \fzand \pr{x}_1 \fzand \pr{x}_2 \fzand \neg \pr{x}_3) \fzor (\pr{x}_1 \fzand \pr{x}_2 \fzand \pr{x}_3) \fzor (\pr{x}_1 \fzand \neg \pr{x}_2) \fzor \neg \pr{x}_1 $. We need an RBM with only 4 hidden units to represent this SDNF.\footnote{Of course, the number of hidden units will grow exponentially with the number of disjuncts in $y$ (typically not allowed in logic programming), e.g. if $y \equiv (a \vee b \vee c)$ then the full DNF will have seven conjunctive clauses.}
 
\section{Reasoning in RBMs}
\label{sub:reasoning}
\subsection{Reasoning as Sampling}

There is a direct relationship between inference in RBMs and logical satisfiability, as follows.
\begin{proposition}
  \label{prop:gibbs_sat}
Let $\mathcal{N}$ be an RBM constructed from a formula $\varphi$. Let $\mathcal{A}$ be a set of indices of variables that have been assigned to either True or False (we use $\vt{x}_\mathcal{A}$ to denote the set $\{x_\alpha|\alpha \in \mathcal{A}\}$). Let $\mathcal{B}$ be a set of indices of variables that have not been assigned a truth-value (we use $\vt{x}_\mathcal{B}$ to denote $ \{x_\beta|\beta \in \mathcal{B}\}$). Performing Gibbs sampling on $\mathcal{N}$ given $\vt{x}_\mathcal{A}$ is equivalent to searching for an assignment of truth-values for $\vt{x}_\mathcal{B}$ that satisfies $\varphi$.
\end{proposition}
\begin{proof}
 Theorem \ref{theorem:prop_rbm} has shown that the truth-value of $\varphi$ is inversely proportional to an RBM's rank function, that is:
 \begin{equation}
   s_\varphi(\vt{x}_\mathcal{B},\vt{x}_\mathcal{A}) \propto -
   min_{\vt{h}}\En(\vt{x}_{\mathcal{B}},\vt{x}_{\mathcal{A}},\vt{h})
 \end{equation}
Therefore, a value of $\vt{x}_\mathcal{B}$ that minimises the energy function also maximises the truth value, because:
 \begin{equation}
 \begin{aligned}
 \vt{x}_\mathcal{B}^* &= \argmin_{\vt{x}_\mathcal{B}} (\min_\vt{h}\En(\vt{x}_\mathcal{B},\vt{x}_\mathcal{A},\vt{h})) \\
 & = \argmin_{\vt{x}_\mathcal{B}} (-s_\varphi(\vt{x}_\mathcal{B},\vt{x}_\mathcal{A}))\\ 
 &= \argmax_{\vt{x}_\mathcal{B}} (s_\varphi(\vt{x}_\mathcal{B},\vt{x}_\mathcal{A}))
 \end{aligned}
 \end{equation}

 Now, we can consider an iterative process to search for truth-values $\vt{x}_\mathcal{B}^*$ by minimising an RBM's energy function. This can be done by using gradient descent to update the values of $\vt{h}$ and then $\vt{x}_\mathcal{B}$ one at a time (similarly to the contrastive divergence algorithm) to minimise $\En(\vt{x}_\mathcal{B},\vt{x}_\mathcal{A},\vt{h})$ while keeping the other variables ($\vt{x}_\mathcal{A}$) fixed. The alternating updates are repeated until convergence. Notice that the gradients amount to:
 \begin{equation}
   \label{eq:en_grad}
   \begin{aligned}
     \frac{\partial -\En(\vt{x}_\mathcal{B},\vt{x}_\mathcal{A},\vt{h})}{\partial h_j} &= \sum_{i\in \mathcal{A}\cup\mathcal{B}} x_iw_{ij} + b_j\\
     \frac{\partial -\En(\vt{x}_\mathcal{B},\vt{x}_\mathcal{A},\vt{h})}{\partial x_\beta} &= \sum_j h_jw_{\beta j} + a_\beta
    \end{aligned}
 \end{equation}

 In the case of Gibbs sampling, given the assigned variables $\vt{x}_\mathcal{A}$, the process starts with a random initialisation of $\vt{x}_\mathcal{B}$, and proceeds to infer values for the hidden units $h_j$ and then the unassigned variables $x_\beta$ in the visible layer of the RBM, using the conditional distributions $h_j \sim p(h_j|\vt{x})$ and $x_\beta \sim p(x_\beta|\vt{h})$, respectively, where $\vt{x}=\{\vt{x}_\mathcal{A},\vt{x}_\mathcal{B}\}$ and 
 
 \begin{equation}
   \label{eq:gibbs}
   \begin{aligned}
     p(h_j|\vt{x}) & = \frac{1}{1+e^{-\frac{1}{\tau}\sum_i x_i w_{ij}+b_j}}  = \frac{1}{1+e^{-\frac{1}{\tau} \frac{\partial -\En(\vt{x}_\mathcal{B},\vt{x}_\mathcal{A},\vt{h})}{\partial h_j}}}\\
     p(x_\beta|\vt{h}) & = \frac{1}{1+e^{-\frac{1}{\tau}\sum_j h_j w_{\beta j}+a_\beta}}  =  \frac{1}{1+e^{-\frac{1}{\tau} \frac{\partial -\En(\vt{x}_\mathcal{B},\vt{x}_\mathcal{A},\vt{h})}{\partial x_\beta}}}\\
   \end{aligned}
  \end{equation}
  
It can be seen from \eqref{eq:gibbs} that the distributions are monotonic functions of the negative energy's gradient over $\vt{h}$ and $\vt{x}_\mathcal{B}$. Therefore, performing Gibbs sampling on those functions can be seen as moving randomly towards a local point of minimum energy, or equivalently to an assignment of truth-values that satisfies the formula. %In the case of $\tau=0$, this process becomes deterministic.
\end{proof}
Since the energy function of the RBM and the satisfiability of the formula are inversely proportional, each step of Gibbs sampling to reduce the energy should intuitively generate a sample that is closer to satisfying the formula.

\subsection{Reasoning as Lowering Free Energy} \label{cond_dis} 

When the number of unassigned variables is not large such that the partition function can be calculated directly, one can infer the assignments of $\vt{x}_\mathcal{B}$ using the conditional distribution:
\begin{equation}
    P(\vt{x}_\mathcal{B}|\vt{x}_\mathcal{A}) = \frac{e^{-\mathcal{F}_{\mathcal{B}}(\vt{x}_\mathcal{A},\vt{x}_\mathcal{B})}}{\sum_{\vt{x}'_\mathcal{B}}e^{\mathcal{F}_{\mathcal{B}}}(\vt{x}_\mathcal{A},\vt{x}'_\mathcal{B})}
\end{equation}

\noindent where $\mathcal{F}_\mathcal{B}= - \sum_j (- \log(1+\exp(c \sum_{i\in \mathcal{A} \cup \mathcal{B}} w_{ij}x_i + b_j)))$ is known as the free energy. The free energy term $- \log(1+\exp(c \sum_{i\in \mathcal{A} \cup \mathcal{B}} w_{ij}x_i + b_j))$ is a negative softplus function scaled by a non-negative value $c$ called {\it confidence value}. It returns a negative output for a positive input and a close-to-zero output for a negative input. One can change the value of $c$ to make the function smooth as shown in Figure \ref{fig:softplus}. 

\begin{figure}[ht]
\centering
\includegraphics[width=0.45\textwidth]{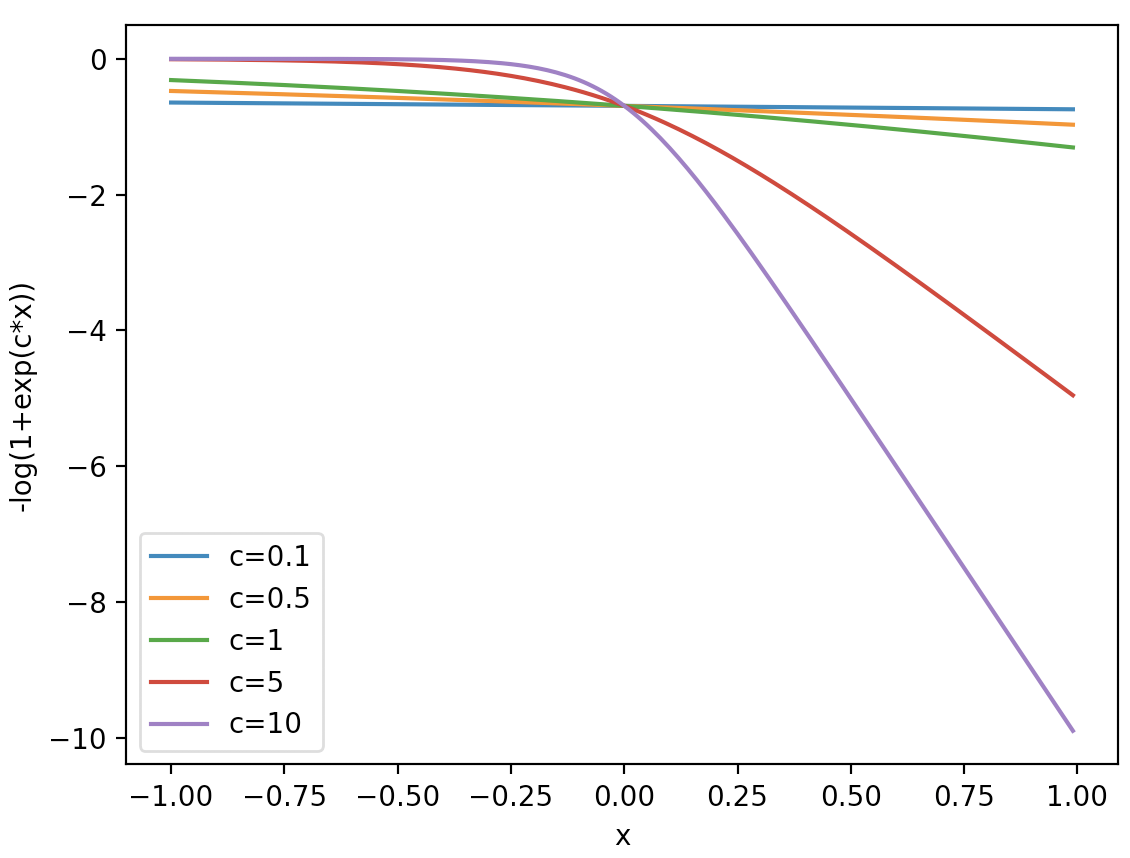}
\caption{Plots of free energy $-\log(1+\exp(cx))$ for various confidence values $c$.}
\label{fig:softplus}
\end{figure}

 Each free energy term is associated with a conjunctive clause in the SDNF through the weighted sum $\sum_{i\in \mathcal{A} \cup \mathcal{B}} w_{ij}x_i + b_j$. Therefore, if a truth-value assignment of  $\vt{x}_\mathcal{B}$ does not satisfy the formula, all energy terms will be close to zero. Otherwise, one free energy term will be $-\log(1+\exp(c\epsilon))$, for a choice of $0<\epsilon<1$ obtained from Theorem \ref{theorem:prop_rbm}. Thus, the more likely a truth assignment is to satisfy the formula, the lower the free energy. 
 \section{Logical Boltzmann Machines}
 \label{LBMsection}
Based on the previous theoretical results, we are now in position to introduce Logical Boltzmann Machines (LBM).  LBM is a neurosymbolic system that uses Restricted Boltzmann Machines for distributed reasoning and learning from data and knowledge.

The LBM system converts any set of formulae $\Phi = \{\varphi_1, ...,\varphi_n\}$ into an RBM by applying Theorem \ref{theorem:prop_rbm} to each formula
$\varphi_i \in \Phi$. In the case of Penalty Logic, formulae are weighted. Given a set of weighted
formulae $\Phi = \{w_1: \varphi_1, ..., w_n:\varphi_n\}$, one can also construct an equivalent RBM where each energy term generated from formula $\varphi_i$ is multiplied by $w_i$. In both cases, the assignments that minimise the energy of the RBM are the assignments that maximise the satifiability of $\Phi$, i.e. the (weighted) sum of the truth-values of the formula.  

\begin{proposition}
  \label{prop:rbm_lprogram} Given a weighted knowledge-base
$\Phi=\{w_1:\varphi_1,...,w_n:\varphi_n\}$, there exists an
equivalent RBM $\mathcal{N}$ such that
$s_\Phi(\vt{x}) = -\frac{1}{\epsilon}min_{\vt{h}}\En(\vt{x},\vt{h})$, where $s_\Phi(\vt{x})$ is the sum of the weights of the formulae in $\Phi$ that are satisfied by assignment $\vt{x}$.
\end{proposition}

A formula $\varphi_i$ can be decomposed into a set of (weighted) conjunctive clauses from its SDNF. If there exist two conjunctive clauses such that one is subsumed by the other then the subsumed clause is removed and the weight of the remaining clause is replaced by the sum of their weights. Identical conjunctive clauses are treated in the same way: one of them is removed and the weights are added. 
%These steps would create a set of generalised and unique weighted conjunctive clauses. Now, 
From Theorem \ref{theorem:prop_rbm}, we know that a conjunctive clause
$\bAnd_{t \in \mathcal{S}_{T_j}}\pr{x}_t \fzand \bAnd_{k \in
\mathcal{S}_{K_j}} \neg \pr{x}_{k}$ is equivalent to an energy term
$e_j(\vt{x},h_j) = -h_j(\sum_{t\in \mathcal{S}_{T_j}} x_t
- \sum_{k \in \mathcal{S}_{K_j}}x_{k} -|\mathcal{S}_{T_{j}}|+ \epsilon)$ where $0<\epsilon<1$. A weighted conjunctive clause $w': \bAnd_{t \in \mathcal{S}_{T_j}}\pr{x}_t \fzand \bAnd_{k \in
\mathcal{S}_{K_j}} \neg \pr{x}_{k}$, therefore, is equivalent to an energy term
 $w' e_j(\vt{x},h_j)$. For each weighted conjunctive clause, we can add a hidden unit $j$ to an RBM with connection weights $w_{tj} = w'$ for all $t\in \mathcal{S}_{T_j}$ and and $w_{kj}=-w'$ for all $k \in
 \mathcal{S}_{K_j}$. The bias for this hidden unit will be $w'(-|\mathcal{S}_{T_{j}}|+\epsilon)$. The weighted knowledge-base and the RBM are equivalent because $s_\Phi(\vt{x})
 \propto -\frac{1}{\epsilon} min_{\vt{h}}\En(\vt{x},\vt{h})$, where $s_\Phi(\vt{x})$ is the
 sum of the weights of the clauses that are satisfied by $\vt{x}$.

\begin{example} (Nixon diamond problem) Consider the following weighted knowledge-base:
{%\footnotesize
\begin{align*}
&1000: \pr{n} \rif \pr{r}  \quad \text{ Nixon is a Republican.} \\       
&1000: \pr{n} \rif \pr{q}  \quad \text{ Nixon is also a Quaker.}\\
&10\text{ }\text{ }\text{ }  : \pr{r} \rif \neg \pr{p} \quad \text{Republicans tend not to be Pacifists.}\\
&10\text{ }\text{ }\text{ }  : \pr{q} \rif \pr{p} \quad \text{Quakers tend to be Pacifists.}
\end{align*}
}
\begin{figure}[ht]
\centering
\includegraphics[width=0.4\textwidth]{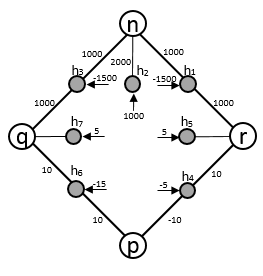}
\caption{The RBM for the Nixon diamond problem. 
With $\epsilon=0.5$, this RBM has energy function:
$
\En = -h_1(1000n+1000r-1500) - h_2(-2000n+1000)-h_3(1000n + 1000q-1500)-h_4(10r - 10p - 5)-h_5(-10r + 5) -h_6(10q + 10p - 15)- h_7(-10q + 5).
$}
\label{fig:rbm_nixon}
%\vskip -.5cm
\end{figure}
Converting all formulae to SDNFs, e.g. $1000:\pr{n} \rif \pr{r} \equiv 1000:(\pr{n} \fzand \pr{r})\fzor (\neg \pr{n})$, produces $8$ conjunctive clauses. After combining the weights of clause ($\neg \pr{n}$) which appears twice, an RBM is created (Figure \ref{fig:rbm_nixon}) from the following unique conjunctive clauses and confidence values:
$ 1000: \pr{n} \fzand \pr{r}, \quad 2000: \neg \pr{n}, \quad 1000: \pr{n} \fzand \pr{q}, \quad 10: \pr{r} \fzand \neg \pr{p}, \quad 10: \neg \pr{r}, \quad 10: \pr{q} \fzand \pr{p}, \quad 10: \neg \pr{q}.
$
\end{example}

\section{Experimental Results}
\label{exp}
\subsection{Reasoning}
In this experiment we apply LBM to effectively search for satisfying truth assignments of variables in large formulae. Let us define a class of formulae:
\begin{equation}
    \varphi \equiv \bAnd_{i=1}^M \pr{x}_i \fzand (\bOr_{j=M+1}^{M+N} \pr{x}_j)
\end{equation}

A formula in this class consists of $2^{M+N}$ possible truth assignments of the variables, with $2^N-1$ of them mapping the formula to $True$ (call this the {\it satisfying set}). Converting to SDNF as done before but now for the class of formulae, we obtain:
\begin{equation}
    \varphi \equiv \bOr_{j=M+1}^{M+N} (\bAnd_{i=1}^M \pr{x}_i \fzand \bAnd_{j'=j+1}^{M+N} \neg \pr{x}_{j'} \fzand \pr{x}_{j} )
\end{equation}

\noindent Applying Theorem \ref{theorem:prop_rbm} to construct an RBM from $\varphi$, we use Gibbs sampling to infer the truth values of all variables. A sample is {\it accepted} as a satisfying assignment if its free energy is lower than or equal to $-\log(1+\exp(c\epsilon)$ with $c=5, \epsilon=0.5$. We evaluate the {\it coverage} and {
\it accuracy} of accepted samples. Coverage is measured as the proportion of the satisfying set that is accepted. In this experiment, this is the number of satisfying assignments in the set of accepted samples divided by $2^N-1$). Accuracy is measured as the percentage of accepted samples that satisfy the formula. 

\begin{figure*}[ht]
\centering
\begin{subfigure}{0.45\textwidth}
\includegraphics[width=.9\textwidth]{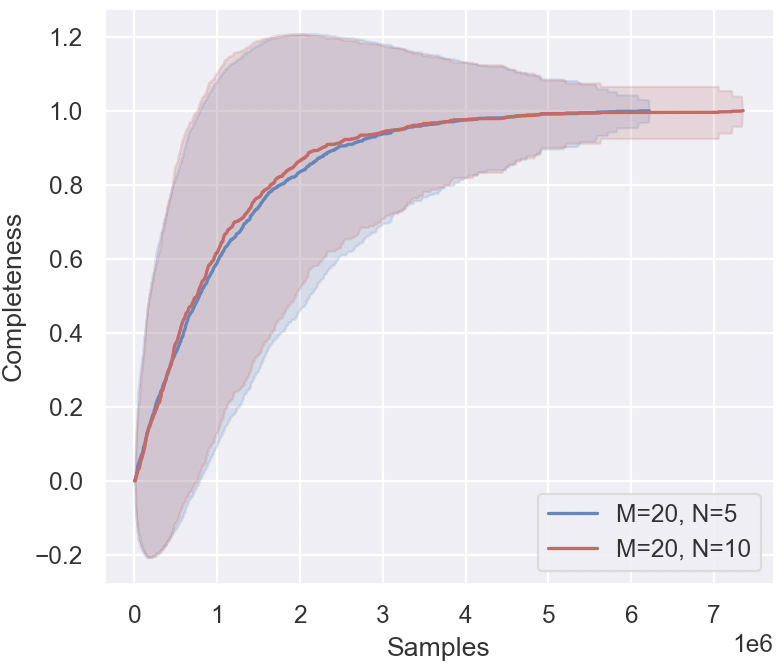}
%\caption{}
%\label{}
\end{subfigure}
\centering
\begin{subfigure}{0.45\textwidth}
\includegraphics[width=.9\textwidth]{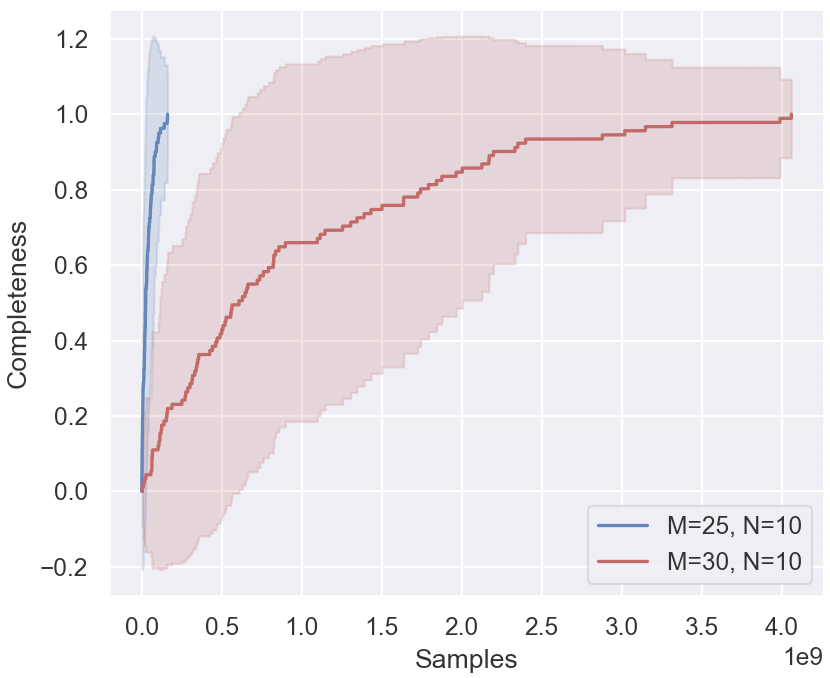}
%\caption{}
%\label{}
\end{subfigure}
\caption{Percentage coverage as sampling progresses in LBM. 100\% coverage is achieved for the class of formulae with different values for M and N averaged over 100 runs. The number of samples needed to achieve $100\%$ coverage is much lower than the number of possible assignments ($2^{M+N}$). For example, when M=20, N=10, all satisfying assignments are found after $\sim 7.5$ million samples are collected, whereas the number of possible assignments is $\sim 1$ billion, a ratio of sample size to the search space of $0.75\%$. The ratio for M=30, N=10 is even lower at $0.37\%$.}
\label{fig:sampling}
\end{figure*}

We test different values of $M\in \{20,25,30\}$ and $N\in\{3,4,5,6,7,8,9,10\}$. LBM achieves $100\%$ accuracy in all cases, meaning that all accepted samples do satisfy the formula.
Figure \ref{fig:sampling} shows the coverage as Gibbs sampling progresses (after each time that a number of samples is collected). Four cases are considered: M=20 and N=5, M=20 and N=10, M=25 and N=10, M=30 and N=10. 
In each case, we run the sampling process 100 times and report the average results with standard deviations. The number of samples needed to achieve $100\%$ coverage is much lower than the number of possible assignments ($2^{M+N}$). For example, when M=20, N=10, all satisfying assignments are found after $\sim 7.5$ million samples are collected, whereas the number of possible assignments is $\sim 1$ billion, producing a ratio of sample size to the search space size of just $0.75\%$. The ratio for M=30, N=10 is even lower at $0.37\%$ w.r.t. $\sim 10^{12}$ possible assignments. As far as we know, this is the first study of reasoning in neurosymbolic AI to produce these results with such low ratios.

\begin{figure}[ht]
    \centering
    \includegraphics[width=0.5\textwidth]{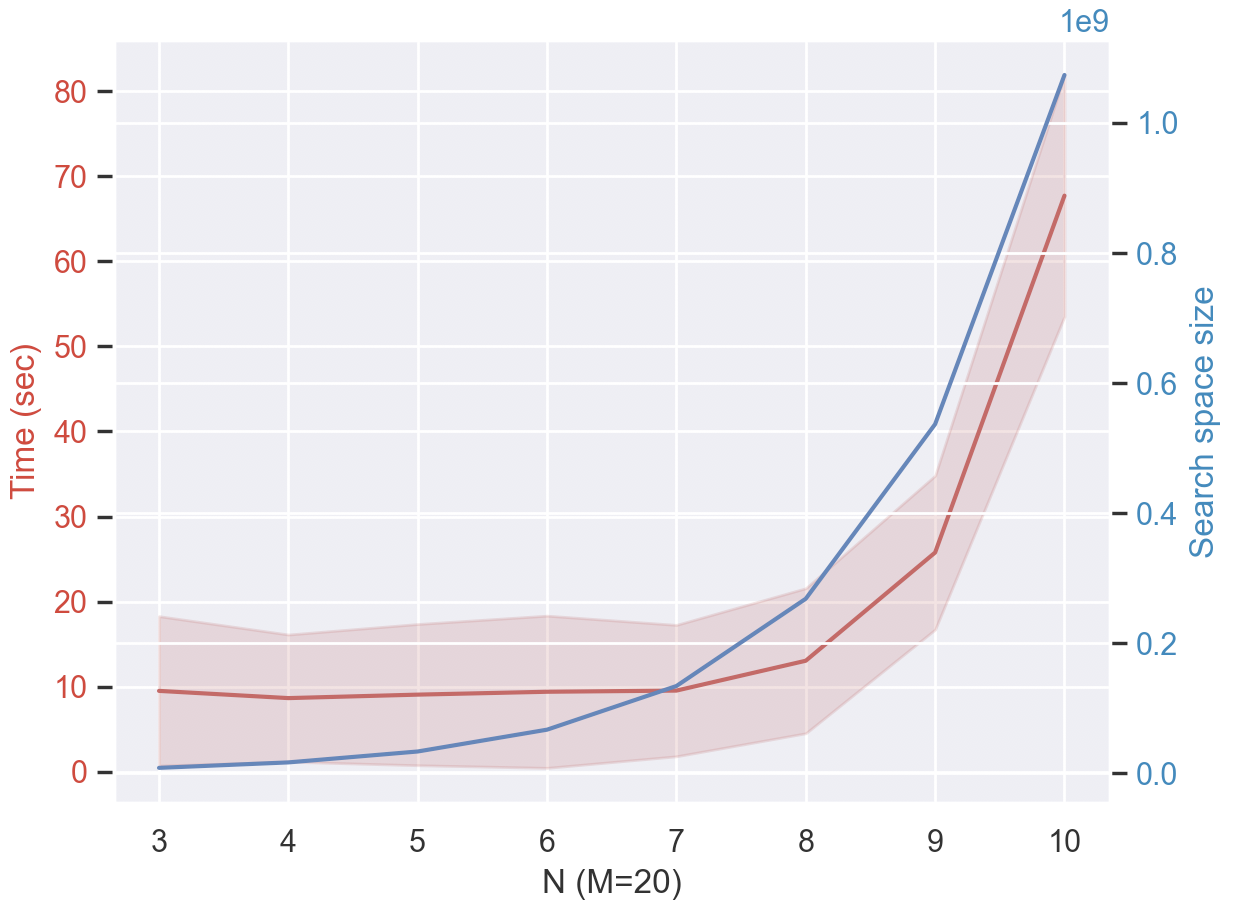}
    \caption{Time taken by LBM to collect all satisfying assignments compared with the size of the search space (number of possible assignments up to $\sim 1$ billion) as N increases from 3 to 10 with fixed M=20. LBM only needs around 10 seconds for $N<=8$, $\sim 25$ seconds for $N=9$, and $\sim 68$ seconds for $N=10$. The curve grows exponentially, similarly to the search space size, but at a much lower scale.}
    \label{fig:time}
\end{figure}

Figure \ref{fig:time} shows the time needed to collect all satisfying assignments for different N in $\{3,4,5,6,7,8,9,10\}$ with $M=20$. LBM only needs around 10 seconds for $N<=8$, $\sim 25$ seconds for $N=9$, and $\sim 68$ seconds for $N=10$. The curve grows exponentially, similarly to the search space size, but at a much lower scale. 

%Propositionalisation translates a relational database into propositional features. It has been shown useful for the integration of learning and reasoning. One of the most popular propositionalisation techniques for neurosymbolic AI, known as {\it Bottom Clause Propositionalisation}, has been applied to several Inductive Logic Programming (ILP) tasks as part of the CILP++ system \cite{Franca_2014} in comparison with ILP state-of-the-art system Aleph \cite{aleph}. 

\subsection{Learning from Data and Knowledge}
In this experiment, we evaluate LBM at learning the same Inductive Logic Programming (ILP) benchmark tasks used by neurosymbolic system CILP++ \cite{Franca_2014} in comparison with ILP state-of-the-art system Aleph \cite{aleph}. An initial LBM is constructed from a subset of the available clauses\footnote{The rest of the clauses is used for training and validation in the usual way, whereby satisfying assignments are selected from each clause as training examples, for instance from clause $y \lif \pr{x}_1 \fzand \neg \pr{x}_2$, assignment $y=1,x_1=1,x_2=0$ is converted into vector $[1,1,0]$.} used as background knowledge, more hidden units with random weights are added to the RBM and the system is trained further from examples, following the methodology used in the evaluation of CILP++. Both confidence values and network weights are free parameters for learning.  

We carry out experiments on 7 data sets: Mutagenesis \cite{Srinivasan_1994}, KRK \cite{Bain_1994}, UW-CSE \cite{Richardson_2006}, and the Alzheimer's benchmark: Amine, Acetyl, Memory and Toxic \cite{King_1995}. For Mutagenesis and KRK, we use $2.5\%$ of the clauses as background knowledge to build the initial LBM. For the larger data sets UW-CSE and Alzheimer's benchmark, we use $10\%$ of clauses as background knowledge. The number of hidden units added to the LBM was chosen arbitrarily to be $50$. For a fair comparison, we also evaluate LBM against a fully-connected RBM with $100$ hidden units so as to offer the RBM more parameters than the LBM since the RBM does not use background knowledge. Both RBM and LBM are trained in discriminative fashion \cite{Larochelle_2012} using the conditional distribution $p(y|\vt{x})$ for inference. The code with these experiments will be made available.  
%\ref{cond_dis}.

The results using 10-fold cross validation are shown in Table \ref{tab:ilp}, except for UW-CSE which use 5 folds. The results for Aleph and CILP++ were collected from \cite{Franca_2014}. It can be seen that LBM has the best performance in 5 out of 7 data sets. In the {\it alz-acetyl} data set, Aleph is better than all other models in this evaluation, and the RBM is best in the {\it alz-amine} data set, despite not using background knowledge (although such knowledge is provided in the form of training examples).  
\begin{table}[ht]
\centering

\begin{tabular}{|l|c|c|c|c|}
\hline
            & { Aleph}    & {CILP++} & RBM     & {LBM}  \\
 \hline
 {\small Mutagenesis}       & 80.85    & 91.70      & 95.55 & \textbf{96.28} \\
 {\small KRK}        & 99.60    & 98.42       & 99.70 & \textbf{99.80} \\
 {\small UW-CSE}     & 84.91    & 70.01 & 89.14      & \textbf{89.43} \\
 {\small alz-amine}  & 78.71    & 78.99 & \textbf{79.13}     & 78.25 \\
 {\small alz-acetyl} & \textbf{69.46}    & 65.47        & 62.93 & 66.82 \\
 {\small alz-memory} & 68.57    & 60.44      & 68.54 & \textbf{71.84} \\
 {\small alz-toxic}  & 80.50    & 81.73     & 82.71  & \textbf{84.95} \\
 \hline
 \hline
\end{tabular}

%\vskip -.2cm
\caption{Cross-validation performance of LBM against Aleph, CILP++ and standard RBM on 7 data sets.}
\label{tab:ilp}
%\vskip -.3cm
\end{table}

\section{Conclusion and Future Work}
\label{sec:discussion}
We introduced an approach and neurosymbolic system to reason about symbolic knowledge at scale in an energy-based neural network. We showed equivalence between minimising the energy of RBMs and satisfiability of Boolean formulae. We evaluated the system at learning and showed its effectiveness in comparison with state-of-the-art approaches. As future work we shall analyse further the empirical results and seek to pinpoint the benefits of LBM at combining reasoning and learning at scale.

% Use \bibliography{yourbibfile} instead or the References section will not appear in your paper
\bibliographystyle{plain}
\bibliography{bibio}

\end{document}